\documentclass[11pt]{article}
\usepackage[T1]{fontenc}
\usepackage{kpfonts}
\usepackage{booktabs}
\usepackage{nicefrac}       
\usepackage{microtype}      
\usepackage{amsmath,amsbsy,amsfonts,amssymb,amsthm,units}
\usepackage{graphicx}
\usepackage{tablefootnote}
\usepackage{color,cases}
\usepackage{tikz}
\usetikzlibrary{calc,shapes,intersections}
\usepackage{wrapfig}
\usepackage{subfigure}
\usepackage[margin=1.1in]{geometry}
\usepackage{xcolor}
\definecolor{DarkGreen}{rgb}{0.1,0.5,0.1}
\definecolor{DarkRed}{rgb}{0.5,0.1,0.1}
\definecolor{DarkBlue}{rgb}{0.1,0.1,0.5}
\definecolor{Black}{rgb}{0.0,0.0,0.0}
\usepackage{hyperref}
\hypersetup{
    colorlinks=true,       
    linkcolor=DarkBlue,          
    citecolor=DarkBlue,        
    filecolor=DarkBlue,      
    urlcolor=DarkBlue,          
    pdftitle={Equality of Opportunity in Supervised Learning},
    pdfauthor={Moritz Hardt, Eric Price, Nathan Srebro},
}
\usepackage{nicefrac}
\usepackage{algorithm}
\usepackage{algorithmic}
\usepackage{xspace}
\usepackage{wrapfig}
\usepackage{multicol}

\usepackage{caption}
\newtheorem{theorem}{Theorem}[section]

\newtheorem{lemma}[theorem]{Lemma}

\newtheorem{corollary}[theorem]{Corollary}

\newtheorem{proposition}[theorem]{Proposition}
\theoremstyle{definition}

\newtheorem{definition}[theorem]{Definition}

\numberwithin{equation}{section}

\newcommand{\IGNORE}[1]{}

\newcommand\E{\mathbb{E}}

\renewcommand{\hat}{\widehat}
\renewcommand{\tilde}{\widetilde}

\newcommand{\defeq}{\stackrel{\small \mathrm{def}}{=}}
\renewcommand{\leq}{\leqslant}
\renewcommand{\le}{\leqslant}

\newcommand{\set}[1]{\left\{#1\right\}}

\newcommand{\Set}[1]{\left\{#1\right\}}

\newcommand\bits{\{0,1\}}

\newcommand{\Yhat}{\widehat{Y}}
\newcommand{\yhat}{\widehat{y}}
\newcommand{\Ytilde}{\widetilde{Y}}
\newcommand{\Rhat}{\widehat{R}}

\newcommand{\ind}{\mathbb{I}}

\def\shownotes{1}  
\ifnum\shownotes=1
\newcommand{\authnote}[2]{{\marginpar{{\color{DarkRed} #1}}$\ll$\textsf{\footnotesize #1 notes: #2}$\gg$}}
\else
\newcommand{\authnote}[2]{}
\fi

\title{Equality of Opportunity in Supervised Learning}
\author{Moritz Hardt\and Eric Price\and Nathan Srebro}

\begin{document}
\maketitle

\begin{abstract}
  We propose a criterion for discrimination against a specified
  sensitive attribute in supervised learning, where the goal is to
  predict some target based on available features. Assuming data about
  the predictor, target, and membership in the protected group are
  available, we show how to optimally \emph{adjust} any learned
  predictor so as to remove discrimination according to our
  definition. Our framework also improves incentives by shifting the
  cost of poor classification from disadvantaged groups to the
  decision maker, who can respond by improving the classification
  accuracy.

  In line with other studies, our notion is \emph{oblivious}: it
  depends only on the joint statistics of the predictor, the target
  and the protected attribute, but not on interpretation of individual
  features. We study the inherent limits of defining and identifying
  biases based on such oblivious measures, outlining what can and
  cannot be inferred from different oblivious tests.

  We illustrate our notion using a case study of FICO credit scores.
\end{abstract}

\section{Introduction}

As machine learning increasingly affects decisions in domains
protected by anti-discrimination law, there is much interest in
algorithmically measuring and ensuring fairness in machine learning.
In domains such as advertising, credit, employment, education, and criminal
justice, machine learning could help obtain more accurate predictions,
but its effect on existing biases is not well understood. Although reliance on
data and quantitative measures can help quantify and eliminate existing biases,
some scholars caution that algorithms can also introduce new biases or
perpetuate existing ones~\cite{BarocasS16}. In May 2014, the Obama
Administration's Big Data Working Group released a report~\cite{Podesta14}
arguing that discrimination can sometimes ``be the inadvertent outcome of the
way big data technologies are structured and used'' and pointed toward ``the
potential of encoding discrimination in automated decisions''. A subsequent
White House report~\cite{WhiteHouse16} calls for ``equal opportunity by design''
as a guiding principle in domains such as credit scoring.

Despite the demand, a vetted methodology for avoiding discrimination against
\emph{protected attributes} in machine learning is lacking. A na\"ive approach might
require that the algorithm should ignore all protected attributes such as
race, color, religion, gender, disability, or family status. However, this idea
of ``fairness through unawareness'' is ineffective due to the existence of
\emph{redundant encodings}, ways of predicting protected attributes from other
features~\cite{PedreshiRT08}.

Another common conception of non-discrimination is \emph{demographic parity}.
Demographic parity requires that a decision---such as accepting or denying a
loan application---be independent of the protected attribute. In the case of a
binary decision~$\Yhat\in\{0,1\}$ and a binary protected
attribute~$A\in\{0,1\},$ this constraint can be formalized by asking that
$\Pr\{\Yhat=1\mid A=0\}=\Pr\{\Yhat=1\mid A=1\}.$ In other words, membership in a
protected class should have no correlation with the decision. Through its
various equivalent formalizations this idea appears in numerous papers.
Unfortunately, as was already argued by Dwork et al.~\cite{DworkHPRZ12}, the
notion is seriously flawed on two counts. First, it doesn't ensure fairness.
Indeed, the notion permits that we accept qualified applicants in the
demographic~$A=0$, but unqualified individuals in~$A=1,$ so long as the
percentages of acceptance match. This behavior can arise naturally, when there
is little or no training data available within $A=1.$ Second, demographic parity
often cripples the utility that we might hope to achieve. Just imagine the
common scenario in which the target variable~$Y$---whether an individual
actually defaults or not---is correlated with~$A.$ Demographic parity would not
allow the ideal predictor $\Yhat = Y,$ which can hardly be considered
discriminatory as it represents the actual outcome. As a result, the loss in
utility of introducing demographic parity can be substantial.

In this paper, we consider non-discrimination from the perspective of supervised
learning, where the goal is to predict a true outcome~$Y$ from features $X$
based on labeled training data, while ensuring they are ``non-discriminatory''
with respect to a specified protected attribute $A$. As in the usual supervised
learning setting, we assume that we have access to labeled training data, in our
case indicating also the protected attribute $A$. That is, to samples from the
joint distribution of $(X,A,Y)$. This data is used to construct a predictor
$\hat{Y}(X)$ or $\hat{Y}(X,A)$, and we also use such data to test whether it is
unfairly discriminatory.

Unlike demographic parity, our notion always allows for the perfectly accurate
solution of~$\Yhat = Y.$ More broadly, our criterion is easier to achieve the more
accurate the predictor $\Yhat$ is, aligning fairness with the central goal in
supervised learning of building more accurate predictors.

The notion we propose is ``oblivious'', in that it is based only on
the joint distribution, or joint statistics, of the true target $Y$,
the predictions $\hat{Y}$, and the protected attribute $A$.  In
particular, it does not evaluate the features in $X$ nor the
functional form of the predictor $\hat{Y}(X)$ nor how it was derived.
This matches other tests recently proposed and conducted, including
demographic parity and different analyses of common risk scores. In
many cases, only oblivious analysis is possible as the functional form
of the score and underlying training data are not public. The only
information about the score is the score itself, which can then be
correlated with the target and protected attribute.  Furthermore, even
if the features or the functional form are available, going beyond
oblivious analysis essentially requires subjective interpretation or
casual assumptions about specific features, which we aim to avoid.

\subsection{Summary of our contributions}
We propose a simple, interpretable, and actionable framework for measuring and
removing discrimination based on protected attributes. We argue that, unlike
demographic parity, our framework provides a meaningful measure of
discrimination, while demonstrating in theory and experiment that we also
achieve much higher utility.  Our key contributions are as follows:
\begin{itemize}
\item We propose an easily checkable and interpretable notion of
  avoiding discrimination based on protected attributes. Our notion
  enjoys a natural interpretation in terms of graphical dependency
  models.  It can also be viewed as shifting the burden of uncertainty
  in classification from the protected class to the decision maker. In
  doing so, our notion helps to incentivize the collection of better
  features, that depend more directly on the target rather then the
  protected attribute, and of data that allows better prediction for
  all protected classes.
\item
We give a simple and effective framework for constructing classifiers
satisfying our criterion from an arbitrary learned predictor. Rather than
changing a possibly complex training pipeline, the result follows via a simple
post-processing step that minimizes the loss in utility.
\item We show that the Bayes optimal non-discriminating (according to
  our definition) classifier is the classifier derived from any Bayes
  optimal (not necessarily non-discriminating) regressor using our
  post-processing step.  Moreover, we quantify the loss that follows
  from imposing our non-discrimination condition in case the score we
  start from deviates from Bayesian optimality. This result helps to
  justify the approach of deriving a fair classifier via
  post-processing rather than changing the original training process.
\item
We capture the inherent limitations of our approach, as well as any other
oblivious approach, through a non-identifiability result showing that different
dependency structures with possibly different intuitive notions of fairness
cannot be separated based on any oblivious notion or test.
\end{itemize}

Throughout our work, we assume a source distribution over $(Y,X,A)$, where $Y$
is the target or true outcome (e.g. ``default on loan''), $X$ are the available
features, and $A$ is the protected attribute. Generally, the features $X$ may be
an arbitrary vector or an abstract object, such as an image. Our work does not
refer to the particular form~$X$ has.

The objective of supervised learning is to construct a (possibly randomized)
predictor $\Yhat=f(X, A)$ that predicts~$Y$ as is typically measured through a
loss function. Furthermore, we would like to require that $\Yhat$ {\em does not
discriminate with respect to $A$}, and the goal of this paper is to formalize
this notion.

\section{Equalized odds and equal opportunity}
We now formally introduce our first criterion.
\begin{definition}[Equalized odds]
We say that a predictor $\Yhat$ satisfies \emph{equalized odds} with
respect to protected attribute~$A$ and outcome~$Y,$ if
$\Yhat$ and $A$ are independent conditional on~$Y.$
\end{definition}
Unlike demographic parity, equalized odds allows $\Yhat$ to depend
on~$A$ but only through the target variable~$Y.$ As such, the
definition encourages the use of features that allow to directly
predict~$Y,$ but prohibits abusing~$A$ as a proxy for~$Y.$

As stated, equalized odds applies to targets and protected attributes taking
values in any space, including binary, multi-class, continuous or structured
settings.  The case of binary random variables $Y,\Yhat$ and~$A$ is of central
importance in many applications, encompassing the main conceptual and technical
challenges. As a result, we focus most of our attention on this case, in which
case equalized odds are equivalent to:
\[
\Pr\Set{\Yhat=1\mid A=0, Y=y} = \Pr\Set{\Yhat=1\mid A=1, Y=y}
,\quad y\in\{0,1\}
\]
For the outcome $y=1,$ the constraint requires that $\Yhat$ has equal \emph{true
  positive rates} across the two demographics $A=0$ and $A=1.$ For
$y=0,$ the constraint equalizes \emph{false positive rates}.  The
definition aligns nicely with the central goal of building highly
accurate classifiers, since $\Yhat=Y$ is always an acceptable
solution. However, equalized odds enforces that the accuracy is
equally high in all demographics, punishing models that perform
well only on the majority.

\subsection{Equal opportunity}
In the binary case, we often think of the outcome $Y=1$ as the ``advantaged''
outcome, such as ``not defaulting on a loan'', ``admission to a college'' or
``receiving a promotion''. A possible relaxation of equalized odds 
is to require non-discrimination only within the ``advantaged'' outcome group.
That is, to require that people who pay back their loan, have an equal
opportunity of getting the loan in the first place (without specifying any
requirement for
those that will ultimately default). This leads to a relaxation of our notion
that we call ``equal opportunity''.
\begin{definition}[Equal opportunity]
We say that a binary predictor~$\Yhat$ satisfies \emph{equal opportunity}
with respect to $A$ and $Y$ if
$\Pr\set{\Yhat=1\mid A=0, Y=1} = \Pr\set{\Yhat=1\mid A=1, Y=1}\,.$
\end{definition}
Equal opportunity is a weaker, though still interesting, notion of
non-discrimination, and thus typically allows for stronger utility as we shall
see in our case study.

\subsection{Real-valued scores}

Even if the target is binary, a real-valued predictive score
$R=f(X,A)$ is often used (e.g. FICO scores for predicting loan
default), with the interpretation that higher values of $R$ correspond
to greater likelihood of $Y=1$ and thus a bias toward predicting
$\Yhat=1$.  A binary classifier $\Yhat$ can be obtained by
thresholding the score, i.e. setting $\Yhat=\mathbb{I}\{R>t\}$ for
some threshold $t$.  Varying this threshold changes the trade-off
between sensitivity (true positive rate) and specificity (true negative rate).

Our definition for equalized odds can be applied also to such score
functions: a score $R$ satisfies equalized odds if $R$ is independent
of $A$ given $Y$.  If a score obeys equalized odds, then any
thresholding $\Yhat=\mathbb{I}\{R>t\}$ of it also obeys equalized odds
(as does any other predictor derived from $R$ alone).  

In Section~\ref{sec:achieving}, we will consider scores that might not satisfy
equalized odds, and see how equalized odds predictors can be derived
from them and the protected attribute $A$, by using different
(possibly randomized) thresholds depending on the value of $A$. 
The same is possible for equality of opportunity without the need for 
randomized thresholds.

\subsection{Oblivious measures}
As stated before, our notions of non-discrimination are \emph{oblivious} 
in the following formal sense.
\begin{definition}
  A property of a predictor $\Yhat$ or score $R$ is said to be
  \emph{oblivious} if it only depends on the joint distribution of
  $(Y,A,\Yhat)$ or $(Y,A,R)$, respectively.
\end{definition}
As a consequence of being oblivious, all the information we need to verify our
definitions is contained in the \emph{joint distribution} of predictor,
protected group and outcome, $(\Yhat, A, Y).$ In the binary case, when $A$ and
$Y$ are reasonably well balanced, the joint distribution of $(\Yhat, A, Y)$ is
determined by $8$ parameters that can be estimated to very high accuracy from
samples. We will therefore ignore the effect of finite sample perturbations and
instead assume that we know the joint distribution of $(\Yhat, A, Y).$

\section{Comparison with related work}

There is much work on this topic in the social sciences and legal scholarship;
we point the reader to Barocas and Selbst~\cite{BarocasS16} for an excellent
entry point to this rich literature. See also the survey by Romei and
Ruggieri~\cite{RomeiR14}, and the references at
\url{http://www.fatml.org/resources.html}.

In its various equivalent notions, demographic parity
appears in many papers, such
as~\cite{CaldersKP09,Zliobaite15,Zafar15} to name a few.
Zemel et al.~\cite{ZemelWSPD13} propose an interesting way of achieving
demographic parity by aiming to learn a representation of the data that is
independent of the protected attribute, while retaining as much information
about the features~$X$ as possible. Louizos et al.~\cite{LouizosSLWZ15} extend
on this approach with deep variational auto-encoders.
Feldman et al.~\cite{FeldmanFMSV15} propose a formalization of ``limiting
disparate impact''. For binary classifiers, the condition states that
$\Pr\set{ \Yhat = 1 \mid A = 0} \le 0.8\cdot \Pr\set{\Yhat = 1 \mid A = 1}.$
The authors argue that this corresponds to the ``80\% rule'' in the legal
literature.  The notion differs from demographic parity mainly in that it
compares the probabilities as a ratio rather than additively, and in that it
allows a one-sided violation of the constraint.

While simple and seemingly intuitive, demographic parity has serious conceptual
limitations as a fairness notion, many of which were pointed out in work of
Dwork et al.~\cite{DworkHPRZ12}. In our experiments, we will see that
demographic parity also falls short on utility.  Dwork et al.~\cite{DworkHPRZ12}
argue that a sound notion of fairness must be \emph{task-specific}, and
formalize fairness based on a hypothetical similarity measure $d(x,x')$
requiring similar individuals to receive a similar distribution over outcomes.
In practice, however, in can be difficult to come up with a suitable metric.
Our notion is task-specific in the sense that it makes critical use of the final
outcome~$Y,$ while avoiding the difficulty of dealing with the features~$X.$

In a recent concurrent work, Kleinberg, Mullainathan and Raghavan~\cite{KleinbergMR16} showed that in general a score that is \emph{calibrated within each group} does \emph{not} satisfy a criterion equivalent to equalized odds for binary predictors. This result highlights that calibration alone does not imply non-discrimination according to our measure. Conversely, achieving equalized odds may in general compromise other desirable properties of a score.

Early work of Pedreshi et al.~\cite{PedreshiRT08} and several follow-up works
explore a logical rule-based approach to non-discrimination. These approaches
don't easily relate to our statistical approach.

\section{Achieving equalized odds and equality of opportunity}
\label{sec:achieving}

We now explain how to find an equalized odds or equal opportunity
predictor $\Ytilde$ derived from a, possibly discriminatory, learned
binary predictor~$\Yhat$ or score~$R.$ We envision that~$\Yhat$ or~$R$
are whatever comes out of the existing training pipeline for the
problem at hand.  Importantly, we do not require changing the training
process, as this might introduce additional complexity, but rather
only a post-learning step.  In particular, we will construct a non-discriminating predictor which is derived from $\Yhat$ or $R$:

\begin{definition}[Derived predictor]
A predictor $\Ytilde$ is \emph{derived from a random variable $R$ and the
protected attribute~$A$} if it is a possibly randomized function of the random
variables $(R, A)$ alone.  In particular, $\Ytilde$ is independent of $X$
conditional on $(R, A).$
\end{definition}
The definition asks that the value of a derived predictor~$\Ytilde$
should only depend on $R$ and the protected attribute, though it may
introduce additional randomness. But the formulation of $\Ytilde$
(that is, the function applied to the values of $R$ and $A$), depends
on information about the joint distribution of $(R, A, Y).$ In other
words, this joint distribution (or an empirical estimate of it) is
required at training time in order to construct the
predictor~$\Ytilde$, but at prediction time we only have access to
values of $(R, A)$.  No further data about the underlying features
$X$, nor their distribution, is required.
\paragraph{Loss minimization.}
It is always easy to construct a trivial predictor satisfying
equalized odds, by making decisions independent of $X,A$ and $R$.  For
example, using the constant predictor $\Yhat=0$ or $\Yhat=1$.  The
goal, of course, is to obtain a {\em good} predictor satisfying the
condition.  To quantify the notion of ``good'', we consider a loss
function $\ell\colon\{0,1\}^2\rightarrow\mathbb{R}$ that takes a pair
of labels and returns a real number $\ell(\yhat,y)\in\mathbb{R}$ which
indicates the loss (or cost, or undesirability) of predicting $\yhat$
when the correct label is~$y.$ Our goal is then to design derived
predictors $\Ytilde$ that minimize the expected loss~$\E\ell(\Ytilde,
Y)$ subject to one of our definitions.

\subsection{Deriving from a binary predictor}
\label{sec:opt-binary}

We will first develop an intuitive geometric solution in the case where we
adjust a binary predictor~$\Yhat$ and $A$ is a binary protected attribute The
proof generalizes directly to the case of a discrete protected attribute with
more than two values. For convenience, we introduce the notation
\begin{equation}\label{eq:gamma}
\gamma_a(\Yhat) \defeq \left(
\Pr\set{ \Yhat = 1 \mid A=a, Y=0},\,
\Pr\set{ \Yhat = 1 \mid A=a, Y=1}
\right)\,.
\end{equation}
The first component of $\gamma_a(\Yhat)$ is the \emph{false positive rate} of
$\Yhat$ within the demographic satisfying $A=a.$ Similarly, the second
component is the \emph{true positive rate} of $\Yhat$ within $A=a.$
Observe that we can calculate $\gamma_a(\Yhat)$ given the
joint distribution of $(\Yhat, A, Y).$
The definitions of equalized odds and equal opportunity can be
expressed in terms of $\gamma(\Yhat)$,  as formalized in the following
straight-forward Lemma:
\begin{lemma}
\label{lem:gamma}
A predictor~$\Yhat$ satisfies:
\begin{enumerate}
\item  equalized odds if and only if
$\gamma_0(\Yhat)=\gamma_1(\Yhat),$ and
\item equal opportunity 
if and only if $\gamma_0(\Yhat)$ and
$\gamma_1(\Yhat)$ agree in the second component, i.e.,
$\gamma_0(\Yhat)_2 = \gamma_1(\Yhat)_2.$
\end{enumerate}
\end{lemma}
For $a\in\{0,1\},$ consider the two-dimensional convex polytope defined as the
convex hull of four vertices:
\begin{align}
P_a(\Yhat)  \defeq
\mathrm{convhull}\Set{ (0,0), \gamma_a(\Yhat), \gamma_a(1-\Yhat),(1,1)}
\end{align}
Our next lemma shows that $P_0(\Yhat)$ and $P_1(\Yhat)$ characterize exactly
the trade-offs between false positives and true positives that we can achieve
with any derived classifier. The polytopes are visualized
in~Figure~\ref{fig:opt-binary}.
\begin{lemma}
\label{lem:derived}
A predictor $\Ytilde$ is derived if and only if for all $a\in\{0,1\},$
we have $\gamma_a(\Ytilde)\in P_a(\Yhat).$
\end{lemma}
\begin{proof}
Since a derived predictor $\Ytilde$ can only depend on $(\Yhat,A)$ and these
variables are binary, the predictor $\Ytilde$ is completely described by four
parameters in~$[0,1]$ corresponding to the probabilities 
$\Pr\Set{\Ytilde =1 \mid \Yhat=\yhat,
A=a}$ for $\yhat,a\in\bits.$
Each of these parameter choices leads to one of the points in $P_a(\Yhat)$ and
every point in the convex hull can be achieved by some parameter setting.
\end{proof}
\begin{figure}
\begin{center}
  \includegraphics[width=\textwidth]{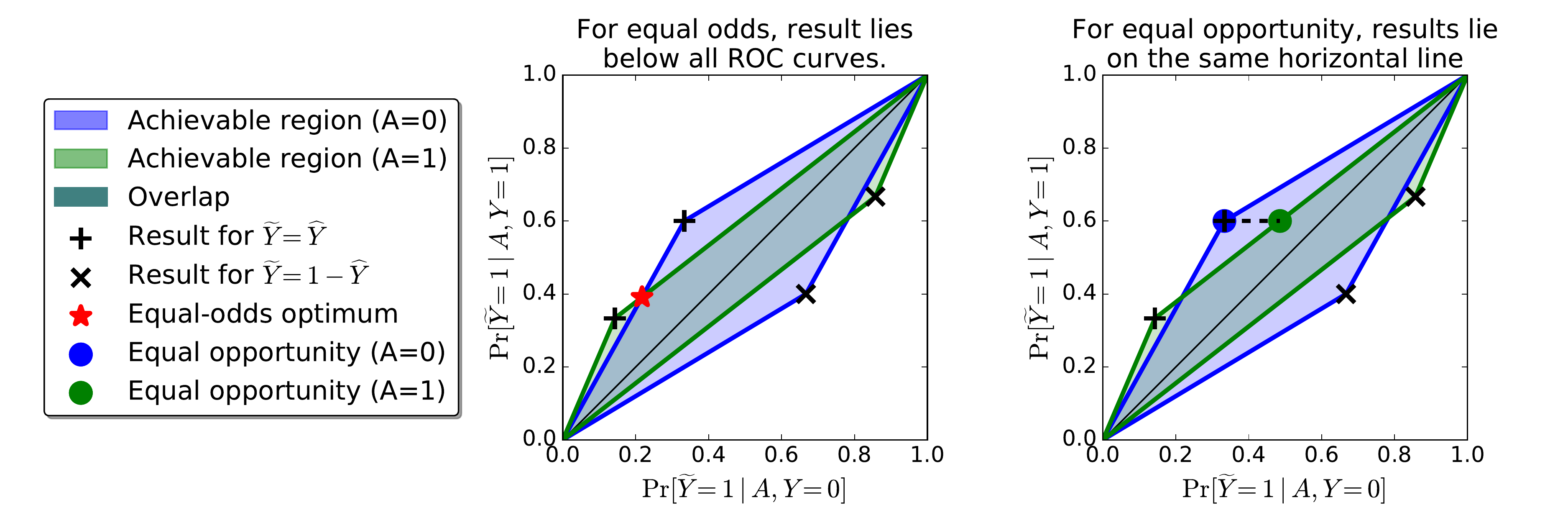}\vspace{-.3cm}
\end{center}
\caption{Finding the optimal equalized odds predictor (left), and equal
opportunity predictor (right).}
\label{fig:opt-binary}\vspace{-.5cm}
\end{figure}
Combining Lemma~\ref{lem:gamma} with Lemma~\ref{lem:derived}, we see that the
following optimization problem gives the optimal derived predictor with
equalized odds:
\begin{align}
\label{eq:opt-binary}
\min_{\Ytilde}\quad
 &\E\ell(\Ytilde, Y)\\
\mathrm{s.t.}\quad&\forall a\in\bits: \gamma_a(\Ytilde)\in P_a(\Yhat)\quad
\tag{\text{derived}}\\
& \gamma_0(\Ytilde) = \gamma_1(\Ytilde)\tag{\text{equalized odds}}
\end{align}
Figure~\ref{fig:opt-binary} gives a simple geometric picture for the solution of
the linear program whose guarantees are summarized next.
\begin{proposition}
\label{prop:opt-binary}
The optimization problem~\eqref{eq:opt-binary} is a linear program in four
variables whose coefficients can be computed from the joint distribution of
$(\Yhat, A, Y).$ Moreover, its solution is an optimal equalized odds predictor
derived from $\Yhat$ and $A.$
\end{proposition}

\begin{proof}[Proof of Proposition~\ref{prop:opt-binary}]
The second claim follows by combining Lemma~\ref{lem:gamma} with
Lemma~\ref{lem:derived}. To argue the first claim, we saw
in the proof of Lemma~\ref{lem:derived} that a derived
predictor is specified by four parameters and the constraint region is an
intersection of two-dimensional linear constraints.
It remains to show that the objective function is a linear function in these
parameters. Writing out the objective, we have
\[
\E\left[\ell(\Ytilde, Y)\right]
= \sum_{y,y'\in\bits} \ell(y,y')\Pr\Set{\Ytilde=y',Y=y}\,.
\]
Further,
\begin{align*}
\Pr\Set{\Ytilde=y',Y=y}
& = \Pr\Set{\Ytilde=y',Y=y\mid \Ytilde=\Yhat}\Pr\Set{\Ytilde=\Yhat}\\
& \quad + \Pr\Set{\Ytilde=y',Y=y\mid \Ytilde\neq\Yhat}\Pr\Set{\Ytilde\neq\Yhat}\\
& = \Pr\Set{\Yhat=y',Y=y}\Pr\Set{\Ytilde=\Yhat}
+ \Pr\Set{\Yhat=1-y',Y=y}\Pr\Set{\Ytilde\neq\Yhat}\,.
\end{align*}
All probabilities in the last line that do not involve~$\Ytilde$ can be computed
from the joint distribution. The probabilities that do involve~$\Ytilde$ are each
a linear function of the parameters that specify~$\Ytilde.$
\end{proof}
The corresponding optimization problem for equation opportunity is the same
except that it has a weaker constraint $\gamma_0(\Ytilde)_2 =
\gamma_1(\Ytilde)_2$.
The proof is analogous to that of Proposition~\ref{prop:opt-binary}.
Figure~\ref{fig:opt-binary} explains the solution geometrically.

\subsection{Deriving from a score function}
\label{sec:opt-thresholds}

We now consider deriving non-discriminating predictors from a real
valued score~$R\in[0,1]$.  The motivation is that in many realistic
scenarios (such as FICO scores), the data are summarized by a
one-dimensional score function and a decision is made based on the
score, typically by thresholding it. Since a continuous statistic can
carry more information than a binary outcome~$Y$, we can hope to
achieve higher utility when working with~$R$ directly, rather then
with a binary predictor $\Yhat$.

A ``protected attribute blind'' way of deriving a binary predictor
from $R$ would be to threshold it, i.e.~using $\Yhat=\ind\set{R>t}$.
If $R$ satisfied equalized odds, then so will such a predictor, and
the optimal threshold should be chosen to balance false positive and
false negatives so as to minimize the expected loss.  When $R$ does
not already satisfy equalized odds, we might need to use different
thresholds for different values of $A$ (different protected groups),
i.e.~$\tilde{Y}=\ind\set{R>t_A}$.  As we will see, even this might not
be sufficient, and we might need to introduce additional randomness as
in the preceding section.

Central to our study is the ROC (Receiver Operator Characteristic)
curve of the score, which captures the false positive and true
positive (equivalently, false negative) rates at different thresholds.
These are curves in a two dimensional plane, where the horizontal axes
is the false positive rate of a predictor and the vertical axes is the
true positive rate.  As discussed in the previous section, equalized
odds can be stated as requiring the true positive and false positive
rates, ($\Pr\set{\Yhat=1 \mid Y=0,A=a},\Pr\set{\Yhat=1
  \mid Y=1,A=a}$), agree between different values of $a$ of the
protected attribute.  That is, that for all values of the protected
attribute, the conditional behavior of the predictor is at exactly the
same point in this space.  We will therefor consider the
$A$-conditional ROC curves
\[
C_a(t)\defeq\left(
\Pr\set{\Rhat > t \mid A=a, Y=0},
\Pr\set{\Rhat > t \mid A=a, Y=1}\right).
\]
Since the ROC curves exactly specify the conditional distributions $R|A,Y$,
a score function obeys equalized odds if and only if the ROC curves for all
values of the protected attribute agree, that is $C_a(t)=C_{a'}(t)$
for all values of $a$ and $t$.  In this case, any thresholding of $R$
yields an equalized odds predictor (all protected groups are at the
same point on the curve, and the same point in false/true-positive
plane).

When the ROC curves do not agree, we might choose different thresholds $t_a$
for the different protected groups.  This yields different points on
each $A$-conditional ROC curve.  For the resulting predictor to satisfy
equalized odds, these must be at the same point in the
false/true-positive plane.  This is possible only at points where all
$A$-conditional ROC curves intersect.  But the ROC curves might not all intersect
except at the trivial endpoints, and even if they do, their point of
intersection might represent a poor tradeoff between false positive
and false negatives.

As with the case of correcting a binary predictor, we can use
randomization to fill the span of possible derived predictors and
allow for significant intersection in the false/true-positive plane.
In particular, for every protected group $a$, consider the convex hull
of the image of the conditional ROC curve:
\begin{equation}\label{eq:conv-hull-threshold}
D_a  \defeq
\mathrm{convhull}\Set{ C_a(t)\colon t\in[0,1]}
\end{equation}
The definition of $D_a$ is analogous to the polytope $P_a$ in the
previous section, except that here we do not consider points below the
main diagonal (line from $(0,0)$ to $(1,1)$), which are worse than
``random guessing'' and hence never desirable for any reasonable loss
function.  

\paragraph{Deriving an optimal equalized odds threshold predictor.}
Any point in the convex hull~$D_a$ represents the false/true positive
rates, conditioned on $A=a$, of a randomized derived predictor based
on $R$.  In particular, since the space is only two-dimensional, such
a predictor~$\tilde Y$ can always be taken to be a mixture of two threshold
predictors (corresponding to the convex hull of two points on the
ROC curve). Conditional on $A=a,$ the predictor $\tilde Y$ behaves as
\[
\tilde{Y} = \ind\set{R>T_a}\,,
\]
where $T_a$ is a randomized threshold assuming the
value~$\underline{t}_a$ with probability~$\underline{p}_a$ and the
value~$\overline{t}_a$ with probability~$\overline{p}_a$.  In other
words, to construct an equalized odds predictor, we should choose a
point in the intersection of these convex hulls,
$\gamma=(\gamma_0,\gamma_1) \in \cap_a D_a$, and then for each
protected group realize the true/false-positive rates $\gamma$ with a
(possible randomized) predictor $\tilde{Y}|(A=a) = \ind\set{R>T_a}$
resulting in the predictor $\tilde{Y} = \Pr\ind\set{R>T_A}$.  For each
group~$a$, we either use a fixed threshold $T_a=t_a$ or a mixture of
two thresholds $\underline{t}_a<\overline{t}_a$.  In the latter case,
if $A=a$ and $R<\underline{t}_a$ we always set $\tilde{Y}=0$, if
$R>\overline{t}_a$ we always set $\tilde{Y}=1$, but if
$\underline{t}_a<R<\overline{t}_a$, we flip a coin and set
$\tilde{Y}=1$ with probability $\underline{p}_a$.
\begin{figure}
\begin{center}
  \includegraphics[width=\textwidth]{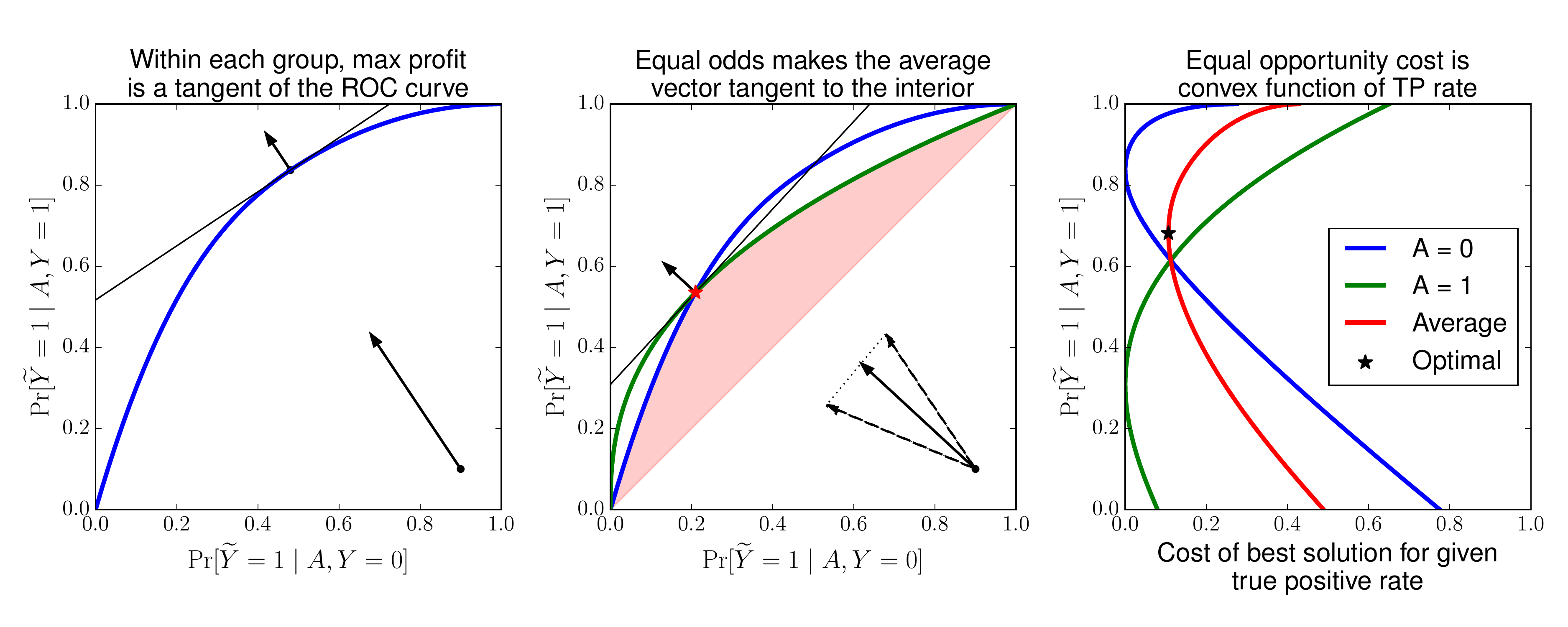}\vspace{-.5cm}
\end{center}
\caption{Finding the optimal equalized odds threshold predictor
  (middle), and equal opportunity threshold predictor (right).  For
  the equal opportunity predictor, within each group the cost for a
  given true positive rate is proportional to the horizontal gap between the ROC
  curve and the profit-maximizing tangent line (i.e., the two curves
  on the left plot), so it is a convex function of the true positive
  rate (right).  This lets us optimize it efficiently with ternary search.}
\label{fig:opt-binary-thresholds}\vspace{-.5cm}
\end{figure}

The feasible set of false/true positive rates of possible equalized
odds predictors is thus the intersection of the areas under the
$A$-conditional ROC curves, and above the main diagonal (see Figure
\ref{fig:opt-binary-thresholds}).  Since for any loss function the optimal
false/true-positive rate will always be on the upper-left boundary of
this feasible set, this is effectively the ROC curve of the equalized odds
predictors.  This ROC curve is the pointwise minimum of all $A$-conditional
ROC curves.  The performance of an equalized odds predictor is thus
determined by the minimum performance among all protected groups.
Said differently, requiring equalized odds incentivizes the learner to
build good predictors for {\em all} classes.
For a given loss function, finding the optimal tradeoff amounts to
optimizing (assuming w.l.o.g.~$\ell(0,0)=\ell(1,1)=0$):
\begin{equation}
  \label{eq:opt-obj-R}
  \min_{\forall a\colon \gamma\in D_a} \gamma_0\ell(1,0)+(1-\gamma_1)\ell(0,1)
\end{equation}
This is no longer a linear program, since $D_a$ are not polytopes, or
at least are not specified as such.  Nevertheless,
\eqref{eq:opt-obj-R} can be efficiently optimized numerically using
ternary search.

\paragraph{Deriving an optimal equal opportunity threshold predictor.}
The construction follows the same approach except that there is one fewer constraint.
We only need to find points on the conditional ROC curves that have the same
true positive rates in both groups. Assuming continuity of the conditional ROC
curves, this means we can always find points on the boundary of the
conditional ROC curves. In this case, no randomization is necessary. The
optimal solution corresponds to two deterministic thresholds, one for each
group. As before, the optimization problem can be solved efficiently using
ternary search over the target true positive value. Here we use, as
Figure~\ref{fig:opt-binary-thresholds} illustrates, that the cost of the best
solution is convex as a function of its true positive rate.

\section{Bayes optimal predictors}
\label{sec:bayes}
In this section, we develop the theory a theory for non-discriminating
Bayes optimal classification.  We will first show that a Bayes optimal
equalized odds predictor can be obtained as an derived threshold
predictor of the Bayes optimal regressor. Second, we quantify the loss
of deriving an equalized odds predictor based on a regressor that
deviates from the Bayes optimal regressor. This can be used to justify
the approach of first training classifiers without any fairness
constraint, and then deriving an equalized odds predictor in a second
step.
\begin{definition}[Bayes optimal regressor]
  Given random variables $(X,A)$ and a target variable~$Y,$ the
  \emph{Bayes optimal regressor} is
  $R=\arg\min_{r(x,a)}\E\left[(Y-r(X,A))^2\right]=r^*(X,A)$ with
  $r^*(x,a) = \E[Y \mid X=x, A=a].$
\end{definition}

The Bayes optimal classifier, for any proper loss, is then a threshold
predictor of~$R,$ where the threshold depends on the loss function
(see, e.g., \cite{Wasserman10}).  We will extend this result to the
case where we additionally ask the classifier to satisfy an oblivious
property, such as our non-discrimination properties.
\begin{proposition}
\label{lem:oblivious-optimal}
\label{prop:oblivious-optimal} For any source distribution over
$(Y,X,A)$ with Bayes optimal regressor $R(X,A)$, any loss function,
and any oblivious property $C$, there exists a predictor $Y^*(R,A)$
such that:
\begin{enumerate}
\item $Y^*$ is an optimal predictor satisfying~$C$.  That is,
  $\E\ell(Y^*,Y)\leq\E\ell(\Yhat,Y)$ for any predictor $\Yhat(X,A)$
  which satisfies~$C$.
\item $Y^*$ is derived from $(R,A).$
\end{enumerate}
\end{proposition}
\begin{figure}[ht]
\begin{center}
  \begin{tikzpicture}[->,auto, semithick, scale=1,every node/.style={scale=1}]
    \tikzset{var/.style={draw=black,circle,minimum size=0.7cm, inner sep=0pt}};
    \node[var] (A) {$A$};
    \node[var, left of=A] (X) {$X$};
    \node[var, right of=A] (X2) {$X'$};
    \node[var] at ($ (A)!.5!(X) +  (0, 1) $) (R) {$R$};
    \node[var] at ($ (A)!.5!(X) -  (0, 1) $) (Yh) {$\Yhat$};
    \node[var, right of=R] (Y) {$Y$};
    \node[var, right of=Yh] (Ys) {$Y^*$};
    {
      \path (X) edge (R)
            edge (Yh)
        (A) edge (R)
            edge (X)
            edge (Yh)
            edge (Ys)
            edge (X2)
        (R) edge (X2)
            edge (Y)
        (X2) edge (Ys)
        (X) edge (Yh)
        ;
      }
    \end{tikzpicture}
\end{center}
    \caption{Graphical model for the proof of Proposition~\ref{prop:oblivious-optimal}.}
\label{fig:graphmodel}
\end{figure}
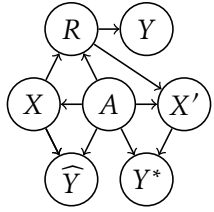
\begin{proof}
  Consider an arbitrary classifier $\Yhat$ on the attributes $(X,A)$,
  defined by a (possibly randomized) function $\Yhat=f(X,A).$ Given
  $(R=r, A=a)$, we can draw a fresh $X'$ from the distribution $(X
  \mid R=r, A=a)$, and set $Y^* = f(X', a)$.  This satisfies (2).
  Moreover, since $Y$ is binary with expectation $R$, $Y$ is
  independent of $X$ conditioned on $(R, A)$.  Hence $(Y, X, R, A)$
  and $(Y, X', R, A)$ have identical distributions, so $(Y^*, A, Y)$ and
  $(\Yhat, A, Y)$ also have identical distributions.  This implies $Y^*$
  satisfies (1) as desired.
\end{proof}

\begin{corollary}[Optimality characterization]
\label{cor:bayes-opt-fair}
An optimal equalized odds predictor can be derived from the Bayes
optimal regressor~$R$ and the protected attribute~$A.$ The same is true for
an optimal equal opportunity predictor.
\end{corollary}

\subsection{Near optimality}
We can furthermore show that if we can approximate the (unconstrained)
Bayes optimal regressor well enough, then we can also construct a
nearly optimal non-discriminating classifier.

To state the result, we introduce the following distance measure on random
variables.
\begin{definition}
We define the \emph{conditional Kolmogorov distance}
between two random variables $R,R'\in[0,1]$ in the same probability space
as~$A$ and~$Y$ as:
\begin{equation}
d_{\mathrm{K}}(R, R')\defeq
\max_{a,y\in\{0,1\}} \sup_{t\in[0,1]}
\left|\Pr\Set{R > t\mid A=a, Y=y}
- \Pr\Set{R' > t\mid A=a, Y=y}
\right|\,.
\end{equation}
\end{definition}
Without the conditioning on $A$ and $Y,$ this definition coincides with the
standard Kolmogorov distance. Closeness in Kolmogorov distance is a rather weak
requirement. We need the slightly stronger condition that the Kolmogorov
distance is small for each of the four conditionings on $A$ and $Y.$ This
captures the distance between the restricted ROC curves, as formalized next.
\begin{lemma}\label{lem:pythagoras}
Let $R, R'\in[0,1]$ be random variables in the same probability space as $A$ and
$Y.$ Then, for any point $p$ on a restricted ROC curve of $R,$ there is a point
$q$ on the corresponding restricted ROC curve of $R'$ such that
$\|p-q\|_2\le\sqrt{2}\cdot d_K(R, R').$
\end{lemma}
\begin{proof}
Assume the point $p$ is achieved by thresholding $R$ at $t\in[0,1].$ Let $q$ be
the point on the ROC curve achieved by thresholding $R'$ at the same threshold
$t'.$ After applying the definition to bound the distance in each coordinate,
the claim follows from Pythagoras' theorem.
\end{proof}
We can now show that an equalized odds predictor derived from a nearly optimal
regressor is still nearly optimal among all equal odds predictors, while
quantifying the loss in terms of the conditional Kolmogorov distance.
\begin{theorem}[Near optimality]
\label{cor:near-optimality}
\label{thm:near-optimality}
Assume that~$\ell$ is a bounded loss function, and let $\Rhat\in[0,1]$ be an
arbitrary random variable. Then, there is an optimal equalized odds
predictor~$Y^*$ and an equalized odds predictor $\Yhat$ derived from~$(\Rhat, A)$
such that
\[
\E\ell(\Yhat, Y)\le \E\ell(Y^*,Y) + 2\sqrt{2}\cdot d_{\mathrm{K}}(\Rhat,R^*)\,,
\]
where $R^*$ is the Bayes optimal regressor.  The same claim is true for equal
opportunity.
\end{theorem}

\begin{proof}[Proof of Theorem~\ref{thm:near-optimality}]
We prove the claim for equalized odds. The case of equal opportunity is
analogous.

Fix the loss function~$\ell$ and the regressor~$\Rhat.$ Take $Y^*$ to be the
predictor derived from the Bayes optimal regressor~$R^*$ and~$A.$ By
Corollary~\ref{cor:bayes-opt-fair}, we know that this is an optimal equalized
odds predictor as required by the lemma.
It remains to construct a derived equalized odds predictor~$\Yhat$ and relate
its loss to that of $Y^*.$

Recall the optimization problem for defining the optimal derived equalized odds predictor.
Let $\widehat D_a$ be the constraint region defined by $\Rhat.$ Likewise,
let $D_a^*$ be
the constraint region under $R^*.$ The optimal classifier~$Y^*$ corresponds to a
point $p^*\in D_0^*\cap D_1^*.$ As a consequence of Lemma~\ref{lem:pythagoras},
we can find (not necessarily identical) points
$q_0\in \widehat D_0$ and $q_1\in \widehat D_1$ such that for all $a\in\{0,1\},$
\[
\|p^*-q_a\|_2\le \sqrt{2}\cdot d_{\mathrm{K}}(\Rhat, R^*)\,.
\]
We claim that this means we can also find a feasible
point $q\in \widehat D_0\cap \widehat D_1$ such that
\[
\|p^*-q\|_2\le 2\cdot d_{\mathrm{K}}(\Rhat, R^*)\,.
\]
To see this, assume without loss of generality that the first coordinate of
$q_1$ is greater than the first coordinate of $q_0,$ and that all points
$p^*,q_0,q_1$ lie
above the main diagonal. By definition of $\widehat D_1,$ we know that the
entire line segment~$L_1$ from $(0, 0)$ to $q_1$ is contained in $\widehat
D_1.$ Similarly, the entire line segment~$L_0$ between $q_0$ and $(1,1)$ is
contained in $\widehat D_0.$ Now, take $q\in L_0\cap L_1.$ By construction,
$q\in \widehat D_0\cap\widehat D_1$ defines a classifier $\Yhat$ derived from
$\Rhat$ and~$A.$ Moreover,
\[
\|p^*-q\|_2^2 \le \|p^*-q_0\|_2^2 + \|p^*-q_0\|_2^2
\le 4\cdot d_{\mathrm{K}}(\Rhat, R^*)^2\,.
\]
Finally, by assumption on the loss function,
there is a vector $v$ with $\|v\|_2\le\sqrt{2}$ such that $\E\ell(\Yhat,
Y)=\langle v, q\rangle$ and $\E\ell(Y^*,Y)=\langle v, p^*\rangle.$ Applying
Cauchy-Schwarz,
\[
\E\ell(\Yhat,Y)-\E\ell(Y^*,Y)
= \langle v, q-p^*\rangle\le \|v\|_2\cdot\|q-p^*\|_2
\le 2\sqrt{2}\cdot d_{\mathrm{K}}(\Rhat, R^*)\,.
\]
This completes the proof.
\end{proof}

\section{Oblivious identifiability of discrimination}
\label{sec:oblivious}

Before turning to analyzing data, we pause to consider to what extent
``black box'' oblivious tests like ours can identify discriminatory
predictions. To shed light on this issue, we introduce two possible scenarios
for the dependency structure of the score, the target and the protected
attribute. We will argue that while these two scenarios can have fundamentally
different interpretations from the point of view of fairness, they can be
indistinguishable from their joint distribution. In particular, no oblivious
test can resolve which of the two scenarios applies.

\paragraph{Scenario I}

\begin{wrapfigure}{r}{0.25\textwidth}
\begin{center}
  \begin{tikzpicture}[->,auto, semithick, scale=1,every node/.style={scale=1}]
    \tikzset{var/.style={draw=black,circle,minimum size=0.7cm, inner sep=0pt}};
    \node[var] (Y) {$Y$};
    \node[var, left of=Y] (A) {$A$};
    \node[var, below of=A] (X1) {$X_1$};
    \node[var, right of=Y] (X2) {$X_2$};
    \node[var, below of=X2] (R2) {$\tilde R$};
    \node[var] at ($ (Y) - (0, 1) $) (R1) {$R^*$};
    {
      \path (A) edge (Y)
            edge (X1)
	(X1) edge (R1)
        (Y) edge (X2)
        (X2) edge (R1)
        (X2) edge (R2)
        ;
    }
    \end{tikzpicture}
\end{center}
\caption{Graphical model for Scenario I.}
\label{fig:scenario1}
\end{wrapfigure}
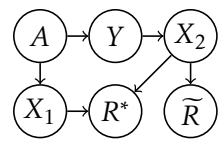

Consider the dependency structure depicted in
Figure~\ref{fig:scenario1}.  Here, $X_1$ is a feature highly (even
deterministically) correlated with the protected attribute $A$, but
independent of the target $Y$ given $A$.  For example, $X_1$ might be
``languages spoken at home'' or ``great great grandfather's
profession''. The target $Y$ has a statistical correlation with the
protected attribute. There's a second real-valued feature $X_2$
correlated with $Y$, but only related to $A$ through $Y$. For example, $X_2$
might capture an applicant's driving record if applying for insurance,
financial activity if applying for a loan, or criminal history in
criminal justice situations. An intuitively ``fair'' predictor here is
to use only the feature $X_2$ through the score $\tilde{R} = X_2$.
The score $\tilde{R}$ satisfies equalized odds, since $X_2$ and $A$ are
independent conditional on~$Y$. Because of the statistical
correlation between $A$ and $Y$, a better statistical predictor, with
greater power, can be obtained by taking into account also the
protected attribute $A$, or perhaps its surrogate $X_1$. The
statistically optimal predictor would have the form
$R^*=r^*_{I}(X_2, X_1)$, biasing the score according to the protected
attribute $A$. The score $R^*$ does {\em not} satisfy equalized odds,
and in a sense seems to be ``profiling'' based on~$A.$

\paragraph{Scenario II}

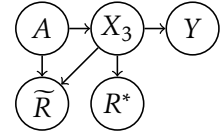
\begin{wrapfigure}{r}{0.25\textwidth}
\begin{center}
  \begin{tikzpicture}[->,auto, semithick, scale=1,every node/.style={scale=1}]
    \tikzset{var/.style={draw=black,circle,minimum size=0.7cm, inner sep=0pt}};
    \node[var] (sX) {$X_3$};
    \node[var, left of=sX] (sA) {$A$};
    \node[var, right of=sX] (sY) {$Y$};
    \node[var, below of=sA] (sR2) {$\tilde R$};
    \node[var, below of=sX] (sR1) {$R^*$};
    {
      \path
        (sA) edge (sX)
        (sX) edge (sY)
        (sX) edge (sR1)
        (sX) edge (sR2)
        (sA) edge (sR2)
        ;
      }
    \end{tikzpicture}
\end{center}
\caption{Graphical model for Scenario II.}
\label{fig:scenario2}
\end{wrapfigure}

Now consider the dependency structure depicted in Figure~\ref{fig:scenario2}. 
Here $X_3$ is a feature, e.g. ``wealth'' or
``annual income'', correlated with the protected attribute $A$ and
directly predictive of the target $Y$.  That is, in this model, the
probability of paying back of a loan is just a function of an
individual's wealth, independent of their race.  Using $X_3$ on its
own as a predictor, e.g.~using the score $R^*=X_3$, does not naturally
seem directly discriminatory.  However, as can be seen from the
dependency structure, this score does {\em not} satisfy equalized
odds. We can correct it to satisfy equalized odds and consider the
optimal non-discriminating predictor $\tilde{R}=\tilde{r}_{II}(X_3,A)$
that does satisfy equalized odds.  If $A$ and $X_3$, and thus $A$ and
$Y$, are positively correlated, then $\tilde{R}$ would depend
inversely on $A$ (see numerical construction below), introducing a
form of ``corrective discrimination'', so as to make $\tilde{R}$ is
independent of $A$ given $Y$ (as is required by equalized odds).

\subsection{Unidentifiability}

The above two scenarios seem rather different.  The optimal score
$R^*$ is in one case based directly on $A$ or its surrogate, and in
another only on a directly predictive feature, but this is not
apparent by considering the equalized odds criterion, suggesting a
possible shortcoming of equalized odds. In fact, as we will now
see, the two scenarios are {\em indistinguishable} using any oblivious
test. That is, no test based only on the target labels, the protected
attribute and the score would give different indications for the
optimal score $R^*$ in the two scenarios. If it were judged unfair in one
scenario, it would also be judged unfair in the other.

We will show this by constructing specific instantiations of the two
scenarios where the joint distributions over $(Y,A,R^*,\tilde{R})$ are
identical. The scenarios are thus unidentifiable based only on these
joint distributions.

We will consider binary targets and protected attributes taking values
in $A,Y\in\{-1,1\}$ and real valued features. We deviate from our convention of
$\{0,1\}$-values only to simplify the resulting expressions.
In Scenario I, let:
\begin{itemize}
\item $\Pr\Set{A=1}=1/2$, and $X_1=A$
\item $Y$ follows a logistic model parametrized based on $A$:
  $\Pr\Set{Y=y\mid A=a} = \frac1{1+\exp(-2ay)},$
\item $X_2$ is Gaussian with mean $Y$: $X_2=Y+{\mathcal N}(0,1)$
\item Optimal unconstrained and equalized odds scores are given by: $R^*=X_1 + X_2=A + X_2,$ and $\tilde{R}=X_2$
\end{itemize}
In Scenario II, let:
\begin{itemize}
\item $\Pr\Set{A=1}=1/2$.
\item $X_3$ conditional on $A=a$ is a mixture of two Gaussians:
  ${\mathcal N}(a+1, 1)$ with weight $\frac1{1+\exp(-2a)}$ and
  ${\mathcal N}(a-1, 1)$ with weight $\frac1{1+\exp(2a)}.$
\item $Y$ follows a logistic model parametrized based on $X_3$: $\Pr\Set{Y=y\mid X_3=x_3}= \frac1{1+\exp(-2yx_3)}.$
\item Optimal unconstrained and equalized odds scores are given by: $R^* = X_3,$  and $\tilde{R}= X_3-A$
\end{itemize}
The following proposition establishes the equivalence
between the scenarios and the optimality of the scores (proof at end of section):
\begin{proposition}\label{prop:unident}
  The joint distributions of $(Y,A,R^*,\tilde{R})$ are identical in
  the above two scenarios.  Moreover, $R^*$ and $\tilde{R}$ are
  optimal unconstrained and equalized odds scores respectively, in
  that their ROC curves are optimal and for any loss function an
  optimal (unconstrained or equalized odds) classifier can be derived
  from them by thresholding.
\end{proposition}
\begin{figure}
\centering
  \begin{tikzpicture}[->,auto, semithick, scale=1,every node/.style={scale=1,,node distance = 14mm}]
    \tikzset{var/.style={draw=black,circle,minimum size=0.7cm, inner sep=0pt}};
    \node[var] (sX3) {$X_3$};
    \node[var, left of=sX3] (sA) {$A$};
    \node[var, right of=sX3] (sY) {$Y$};
    \node[var, below of=sA] (sX1) {$X_1$};
    \node[var, below of=sX3] (sX2) {$X_2$};
    {
      \path
        (sA) edge (sX3)
        (sX3) edge (sY)
        (sA) edge (sX1)
        (sY) edge(sX2)
        (sX1) edge(sX2)
        ;
      }
    
    \end{tikzpicture}
  \qquad \qquad  \qquad
  \begin{tikzpicture}[->,auto, semithick, scale=1,every node/.style={scale=1,node distance = 14mm}]
    \tikzset{var/.style={draw=black,circle,minimum size=0.7cm, inner sep=0pt}};
    \node[var] (sA) {$A$};
    \node[var, left of=sA] (sX1) {$X_1$};
    \node[var, right of=sA] (sY) {$Y$};
    \node[var, right of=sY] (sX2) {$X_2$};
    \node[var, below of=sX1] (sX3) {$X_3$};
    {
       \path
        (sA) edge (sX1)
        (sA) edge (sY)
        (sY) edge (sX2)
        (sX2) edge(sX3)
        (sX1) edge(sX3)
        ;
      }
    \end{tikzpicture}

    \caption{Two possible directed dependency structures for the
      variables in scenarios I and II.  The undirected (infrastructure
      graph) versions of both graphs are also possible.}\label{fig:twooptions}
\end{figure}
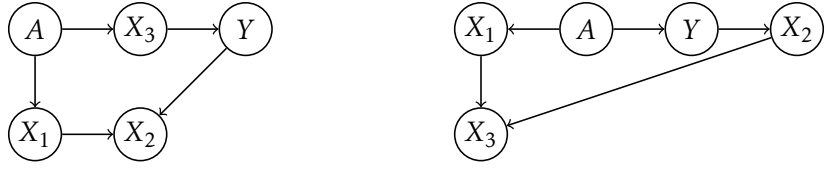
Not only can an oblivious test (based only on $(Y,A,R)$) not
distinguish between the two scenarios, but even having access to the
features is not of much help.  Suppose we have access to all three
feature, i.e.~to a joint distribution over $(Y,A,X_1,X_2,X_3)$---since
the distributions over $(Y,A,R^*,\tilde{R})$ agree, we can construct
such a joint distribution with $X_2=\tilde{R}$ and $X_3=\tilde{R}$.
The features are correlated with each other, with $X_3=X_1+X_2$.
Without attaching meaning to the features or making causal assumptions
about them, we do not gain any further insight on the two scores.  In
particular, both causal structures depicted in Figure~\ref{fig:twooptions} are
possible.

\subsection{Comparison of different oblivious measures}
It is interesting to consider how different oblivious measures apply
to the scores $\tilde{R}$ and $R^*$ in these two scenarios.

As discussed in Section \ref{sec:opt-thresholds}, a score satisfies
equalized odds iff the conditional ROC curves agree for both values of
$A$, which we refer to as having \emph{identical} ROC curves.
\begin{definition}[Identical ROC Curves]
We say that a score $R$ has \emph{identical conditional ROC curves} if
$C_a(t)=C_{a'}(t)$
for all groups of $a,a'$ and all $t\in\mathbb{R}$.
\end{definition}
In particular, this property is achieved by an equalized odds 
score~$\tilde{R}$.
Within each protected group, i.e.~for each value $A=a$, the score
$R^*$ differs from $\tilde{R}$ by a fixed monotone transformation,
namely an additive shift $R^*=\tilde{R}+A$.  Consider a derived
threshold predictor $\hat{Y}(\tilde{R})=\ind\set{\tilde{R}>t}$ based
on $\tilde{R}$.  Any such predictor obeys equalized odds.  We can also
derive the same predictor deterministically from $R^*$ and $A$ as
$\hat{Y}(R^*,A)=\ind\set{R^*>t_A}$ where $t_A=t-A$.  That is, in our
particular example, $R^*$ is special in that optimal equalized odds
predictors can be derived from it (and the protected attribute $A$)
deterministically, without the need to introduce randomness as in
Section \ref{sec:opt-thresholds}.  In terms of the $A$-conditional ROC
curves, this happens because the images of the conditional ROC curves
$C_0$ and $C_1$ overlap, making it possible to choose points in the
true/false-positive rate plane that are on both ROC curves.  However,
the same point on the conditional ROC curves correspond to different
thresholds!  Instead of $C_0(t)=C_1(t)$, for $R^*$ we have
$C_0(t)=C_1(t-1)$.  We refer to this property as ``matching''
conditional ROC curves:
\begin{definition}[Matching ROC curves]
We say that a score $R$ has \emph{matching conditional ROC curves} if
the images of all $A$-conditional ROC curves are the same, i.e.,
for all groups $a,a',$
$\set{ C_a(t) \colon t\in\mathbb{R} } = \set{ C_{a'}(t) \colon
  t\in\mathbb{R} }.$
\end{definition}
Having matching conditional ROC curves corresponds to being
deterministically correctable to be non-discriminating: If a predictor
$R$ has matching conditional ROC curves, then for any loss function
the optimal equalized odds derived predictor is a deterministic
function of $R$ and $A$.  But as our examples show, having matching
ROC curves does not at all mean the score is itself
non-discriminatory: it can be biased according to $A$, and a
(deterministic) correction might be necessary in order to ensure
equalized odds.

Having identical or matching ROC curves are properties of the
conditional distribution $R|Y,A$, also referred to as ``model
errors''.  Oblivious measures can also depend on the conditional
distribution $Y|R,A$, also referred to as ``target population errors''.
In particular, one might consider the following property:
\begin{definition}[Matching frequencies]
  We say that a score $R$ has \emph{matching conditional frequencies},
  if for all groups $a,a'$ and all scores $t,$ we have
\[
\Pr\Set{Y=1\mid R=t, A=a} = \Pr\Set{Y=1\mid R=t, A=a'}\,.
\]
\end{definition}
Matching conditional frequencies state that at a given score, both
groups have the same probability of being labeled positive.  The
definition can also be phrased as requiring that the conditional
distribution $Y|R,A$ be independent of $A$.  In other words, having
matching conditional frequencies is equivalent to $A$ and $Y$ being independent
conditioned on $R$.  The corresponding dependency structure is
$Y-R-A$.  That is, the score $R$ includes all possible information the
protected attribute can provide on the target $Y$.  Indeed having
matching conditional frequencies means that the score is in a sense
``optimally dependent'' on the protected attribute $A$.  Formally, for
any loss function the optimal (unconstrained, possibly discriminatory)
derived predictor $\Yhat(R,A)$ would be a function of $R$ alone, since
$R$ already includes all relevant information about $A$.  In
particular, an unconstrained optimal score, like $R^*$ in our
constructions, would satisfy matching conditional frequencies.  Having
matching frequencies can therefore be seen as a property indicating
utilizing the protected attribute for optimal predictive power,
rather then protecting discrimination based on it.

It is also worth noting the similarity between matching frequencies
and a binary predictor $\Yhat$ having equal conditional precision,
that is $\Pr\Set{Y=1\mid \Yhat=\hat{y}, A=a} = \Pr\Set{Y=1\mid
  \Yhat=\hat{y}, A=a'}$.  Viewing $\Yhat$ as a score that takes two
possible values, the notions agree. But $R$ having matching
conditional frequencies does {\em not} imply the threshold predictors
$\Yhat(R)=\ind\set{R>t}$ will have matching precision---the
conditional distributions $R|A$ might be different, and these are
involved in marginalizing over $R>t$ and $R\leq t$.

To summarize, the properties of the scores in our scenarios are:
\begin{itemize}
\item $R^*$ is optimal based on the features and protected attribute,
  without any constraints.
\item $\tilde{R}$ is optimal among all equalized odds scores.
\item $\tilde R$ does satisfy equal odds, $R^*$ does not satisfy equal odds.
\item $\tilde R$ has identical (thus matching) ROC curves, 
  $R^*$ has matching but non-identical ROC curves.
\item $R^*$ has matching conditional frequencies, while $\tilde{R}$
  does not.
\end{itemize}

\subsubsection*{Proof of Proposition \ref{prop:unident}}

  First consider Scenario I.  The score $\tilde{R}=X_2$ obeys equalized
  odds due to the dependency structure.  More broadly, if a score
  $R=f(X_2,X_1)$ obeys equalized odds, for some randomized function
  $f$, it cannot depend on $X_1$: conditioned on $Y$, $X_2$ is
  independent of $A=X_1$, and so any dependency of $f$ on $X_1$ would
  create a statistical dependency on $A=X_1$ (still conditioned on
  $Y$) which is not allowed.  We can verify that
  $\Pr\set{Y=y\mid X_2=x_2} \propto \Pr\set{Y=y}\Pr\set{X_2=x_2\mid Y=y} \propto
  \exp(2y x_2)$ which is monotone in $X_2$, and so for any loss
  function we would just want to threshold $X_2$ and any function
  monotone in $X_2$ would make an optimal equalized odds predictor.

To obtain the optimal unconstrained score consider
\begin{align*}
\Pr\set{Y=y\mid X_1=x_1,X_2=x_2} & \propto
\Pr\set{A=x_1}\Pr\set{Y=y\mid A=x_1}\Pr\set{X_2=x_2\mid Y=y} \\
&\propto\exp(2y(x_1+x_2)).
\end{align*}
That is, optimal classification only depends on
$x_1+x_2$ and so $R^*=X_1+X_2$ is optimal.

Turning to scenario II, since $P(Y|X_3)$ is monotone in $X_3$, any
monotone function of it is optimal (unconstrained), and the dependency
structure implies its optimal even if we allow dependence on $A$.
Furthermore, the conditional distribution
$Y|X_3$ matched that of  $Y|R^*$ from scenario I since again we have
$\Pr\set{Y=y|X_3=x_3}\propto \exp(2 y x_3)$ by construction.  Since we
defined $R^*=X_3$, we have that the conditionals $R^*|Y$ match.  We
can also verify that by construction $X_3|A$ matches $R^*|A$ in
scenario I.  Since in scenario I, $R^*$ is optimal even dependent on
$A$, we have that $A$ is independent of $Y$ conditioned on $R^*$, as
in scenario II when we condition on $X_3=R^*$.  This establishes the
joint distribution over $(A,Y,R^*)$ is the same in both scenarios.
Since $\tilde{R}$ is the same deterministic function of $A$ and $R^*$
in both scenarios, we can further conclude the joint distributions
over $A,Y,R^*$ and $\tilde{R}$ are the same.  Since equalized odds is
an oblivious property, once these distributions match, if $\tilde{R}$
obeys equalized odds in scenario I, it also obeys it in scenario II.

\section{Case study: FICO scores}
\label{sec:cases}
We examine various fairness measures in the context of FICO scores
with the protected attribute of race.  FICO scores are a proprietary
classifier widely used in the United States to predict credit
worthiness.  Our FICO data is based on a sample of $301536$ TransUnion
TransRisk scores from 2003~\cite{fedreserve}.  These scores, ranging
from $300$ to $850$, try to predict credit risk; they form our
score $R$.  People were labeled as in \emph{default} if they
failed to pay a debt for at least $90$ days on at least one account in
the ensuing 18-24 month period; this gives an outcome $Y$.  Our
protected attribute $A$ is race, which is restricted to four values:
Asian, white non-Hispanic (labeled ``white'' in figures), Hispanic,
and black. FICO scores are complicated proprietary classifiers based
on features, like number of bank accounts kept, that could interact
with culture---and hence race---in unfair ways.
\begin{figure}[h!]
  \centering
  \includegraphics[width=\textwidth]{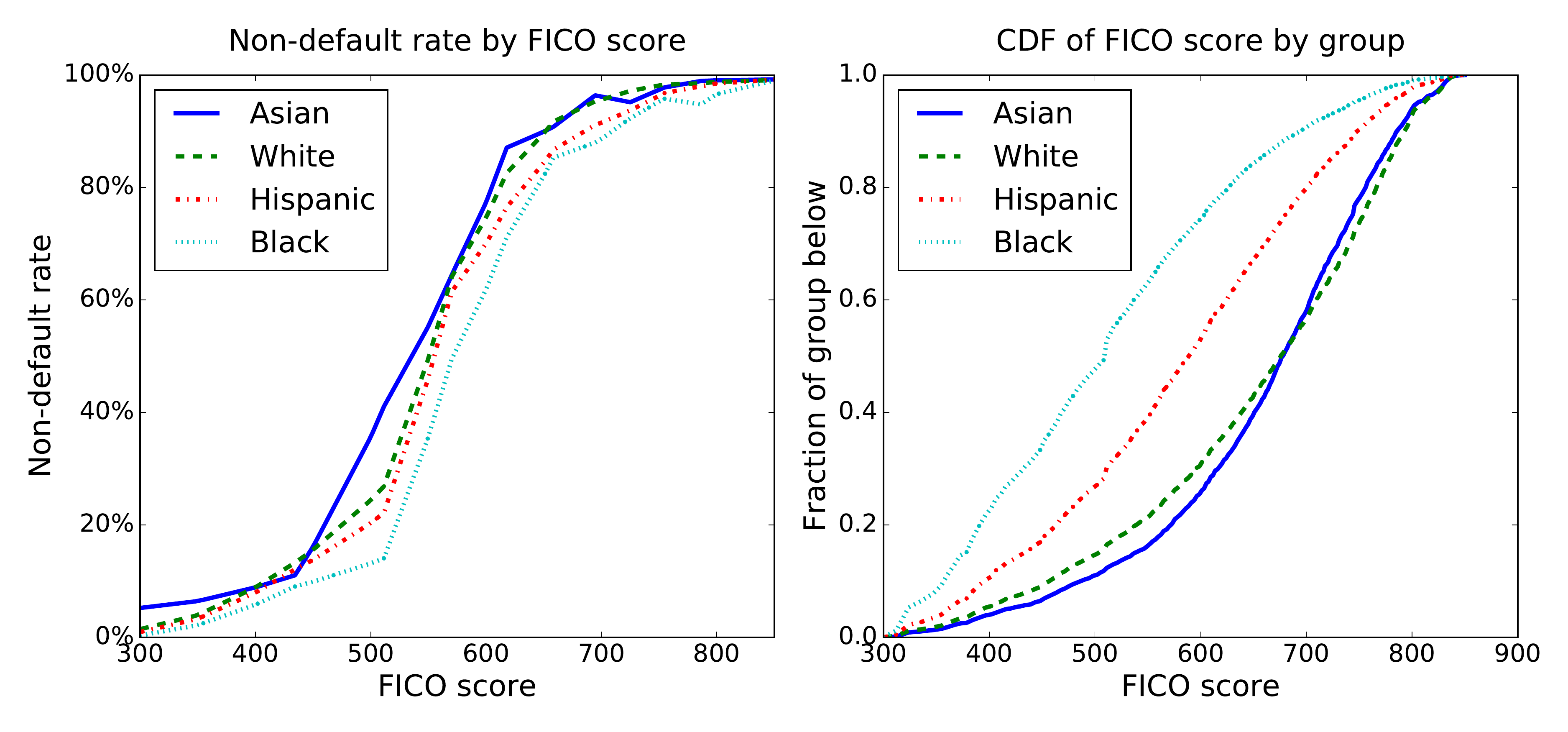}
  \caption{These two marginals, and the number of people per group,
    constitute our input data.}
  \label{fig:fico-marginals}
\end{figure}
A credit score cutoff of 620 is commonly used for prime-rate
loans\footnote{\url{http://www.creditscoring.com/pages/bar.htm}
  (Accessed: 2016-09-20)}, which corresponds to an any-account default
rate of 18\%.  Note that this measures default on \emph{any} account
TransUnion was aware of; it corresponds to a much lower ($\approx
2\%$) chance of default on individual new loans.  To illustrate the
concepts, we use any-account default as our target~$Y$---a higher
positive rate better illustrates the difference between equalized odds
and equal opportunity.

We therefore consider the behavior of a lender who makes money on
default rates below this, i.e., for whom whom false positives (giving
loans to people that default on any account) is 82/18 as expensive as
false negatives (not giving a loan to people that don't default). The
lender thus wants to construct a predictor $\Yhat$ that is optimal
with respect to this asymmetric loss.  A typical classifier will pick
a threshold per group and set $\Yhat = 1$ for people with FICO scores
above the threshold for their group.  Given the marginal distributions
for each group (Figure~\ref{fig:fico-marginals}), we can study the
optimal profit-maximizing classifier under five different constraints
on allowed predictors:
\begin{figure}
  \centering
  \includegraphics[width=0.9\textwidth]{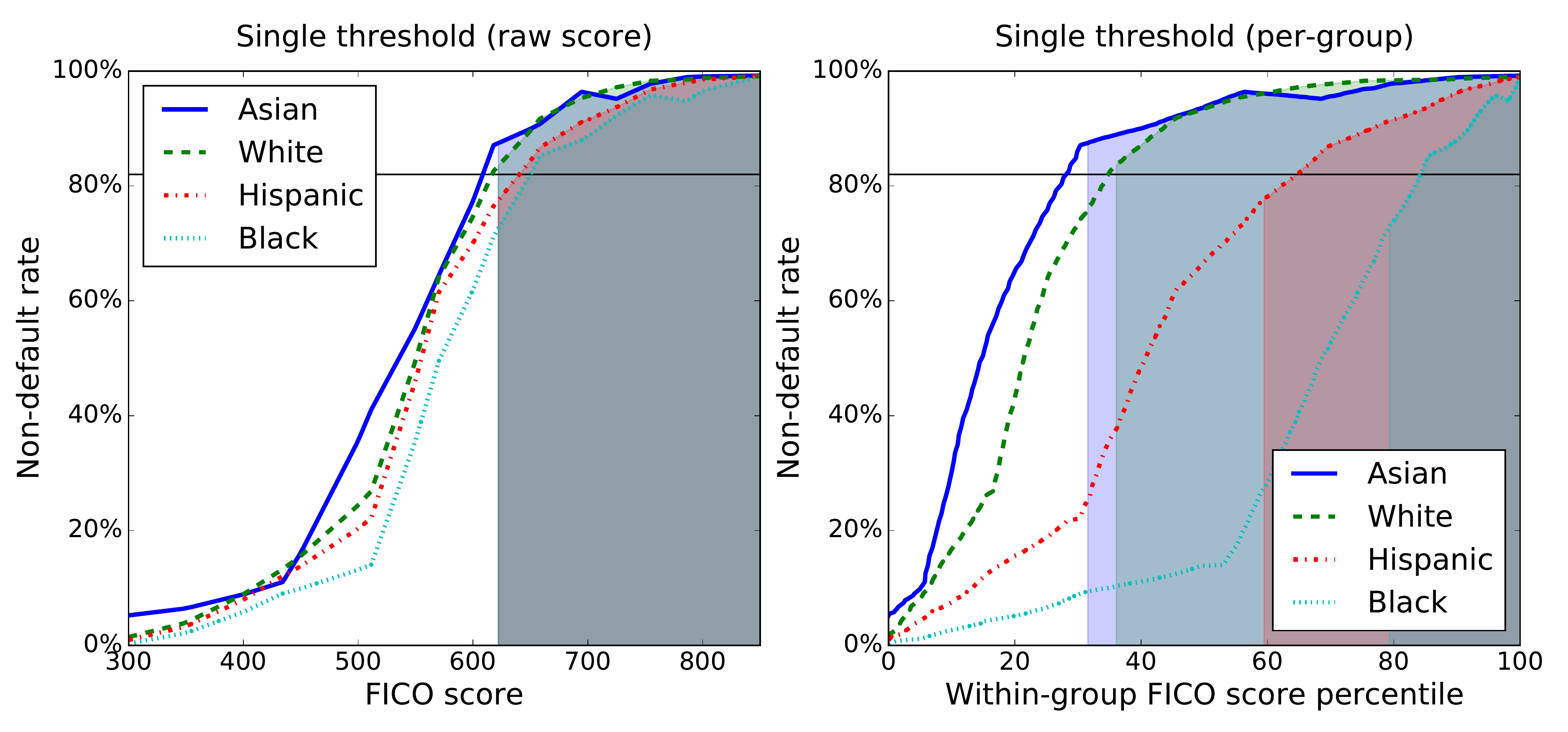}
  \caption{The common FICO threshold of 620 corresponds to a
    non-default rate of 82\%.  Rescaling the $x$ axis to represent the
    within-group thresholds (right), $\Pr[\Yhat=1\mid Y=1,A]$ is the
    fraction of the area under the curve that is shaded. This means
    black non-defaulters are much less likely to qualify for loans
    than white or Asian ones, so a race blind score threshold violates
    our fairness definitions.}
  \label{fig:fico-fixed}
\end{figure}
\begin{itemize}
\item \textbf{Max profit} has no fairness constraints, and will pick
  for each group the threshold that maximizes profit.  This is the
  score at which 82\% of people in that group do not default.
\item \textbf{Race blind} requires the threshold to be the same for
  each group.  Hence it will pick the single threshold at which 82\%
  of people do not default overall, shown in
  Figure~\ref{fig:fico-fixed}.
\item \textbf{Demographic parity} picks for each group a threshold
  such that the fraction of group members that qualify for loans is the same.
\item \textbf{Equal opportunity} picks for each group a threshold such
  that the fraction of \emph{non-defaulting} group members that qualify for loans is the
  same.
\item \textbf{Equalized odds} requires both the fraction of non-defaulters
  that qualify for loans and the fraction of defaulters that qualify
  for loans to be constant across groups.  This cannot be achieved with
  a single threshold for each group, but requires randomization.
  There are many ways to do it; here, we pick \emph{two} thresholds
  for each group, so above both thresholds people always qualify and
  between the thresholds people qualify with some probability.
\end{itemize}
We could generalize the above constraints to allow non-threshold
classifiers, but we can show that each profit-maximizing classifier
will use thresholds. As shown in Section~\ref{sec:achieving}, the
optimal thresholds can be computed efficiently; the results are shown
in Figure~\ref{fig:fico-thresholds}.
\begin{figure}[h!]
  \centering
  \includegraphics[width=\textwidth]{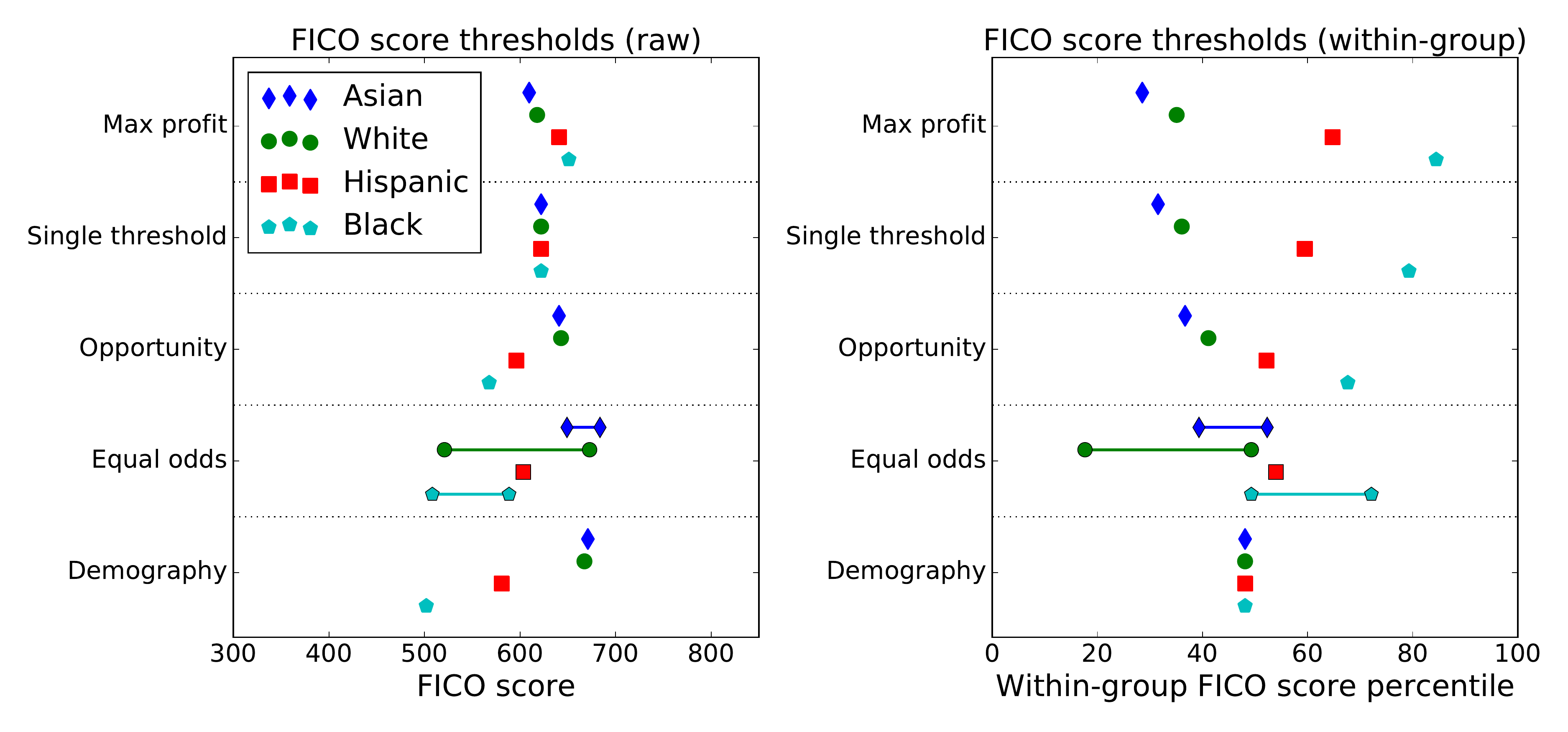}
  \caption{FICO thresholds for various definitions of fairness.
    The equal odds method does not give a single
    threshold, but instead $\Pr[\Yhat = 1 \mid R,A]$ increases over some
    not uniquely defined range; we pick the one containing the fewest people.
    Observe that, within each race, the equal opportunity threshold
    and average equal odds threshold lie between the max profit
    threshold and equal demography thresholds.
  }
  \label{fig:fico-thresholds}
\end{figure}
Our proposed fairness
definitions give thresholds between those of max-profit/race-blind
thresholds and of demographic parity.
\begin{figure}
  \centering
  \includegraphics[width=\textwidth]{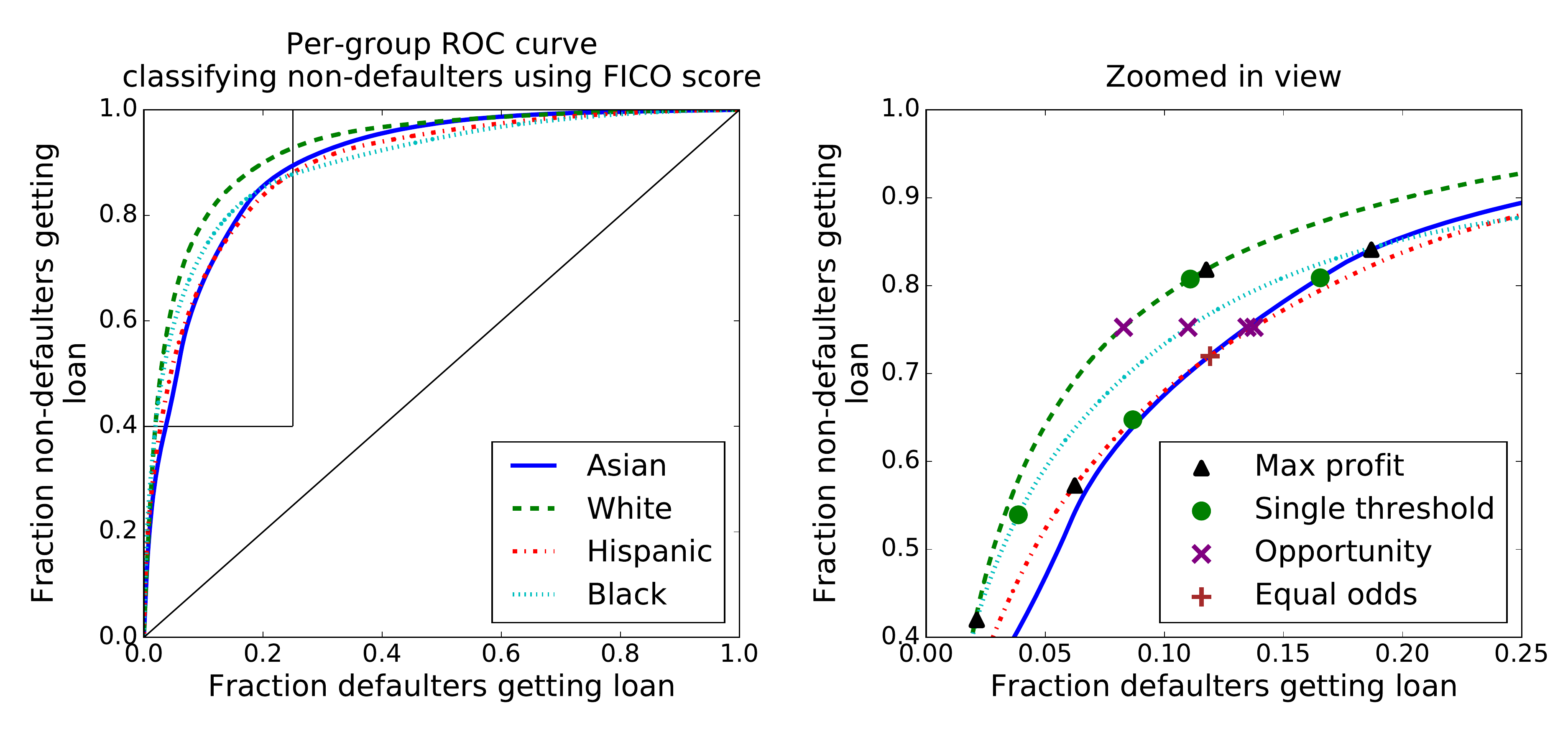}
  \caption{The ROC curve for using FICO score to identify
    non-defaulters.  Within a group, we can achieve any convex
    combination of these outcomes.  Equality of opportunity picks
    points along the same horizontal line.  Equal odds picks a point
    below all lines.}
  \label{fig:fico-roc}
\end{figure}
Figure~\ref{fig:fico-roc} plots the ROC curves for each group.  It
should be emphasized that differences in the ROC curve do not indicate
differences in default behavior but rather differences in prediction
accuracy---lower curves indicate FICO scores are less predictive for
those populations.  This demonstrates, as one should expect, that the
majority (white) group is classified more accurately than minority
groups, even over-represented minority groups like Asians.

The left side of Figure~\ref{fig:fico-profit} shows the fraction of
people that wouldn't default that would qualify for loans by the
various metrics.  Under max-profit and race-blind thresholds, we find
that black people that would not default have a significantly harder
time qualifying for loans than others.  Under demographic parity, the
situation is reversed.
\begin{figure}[h!]
  \centering
  \includegraphics[width=\textwidth]{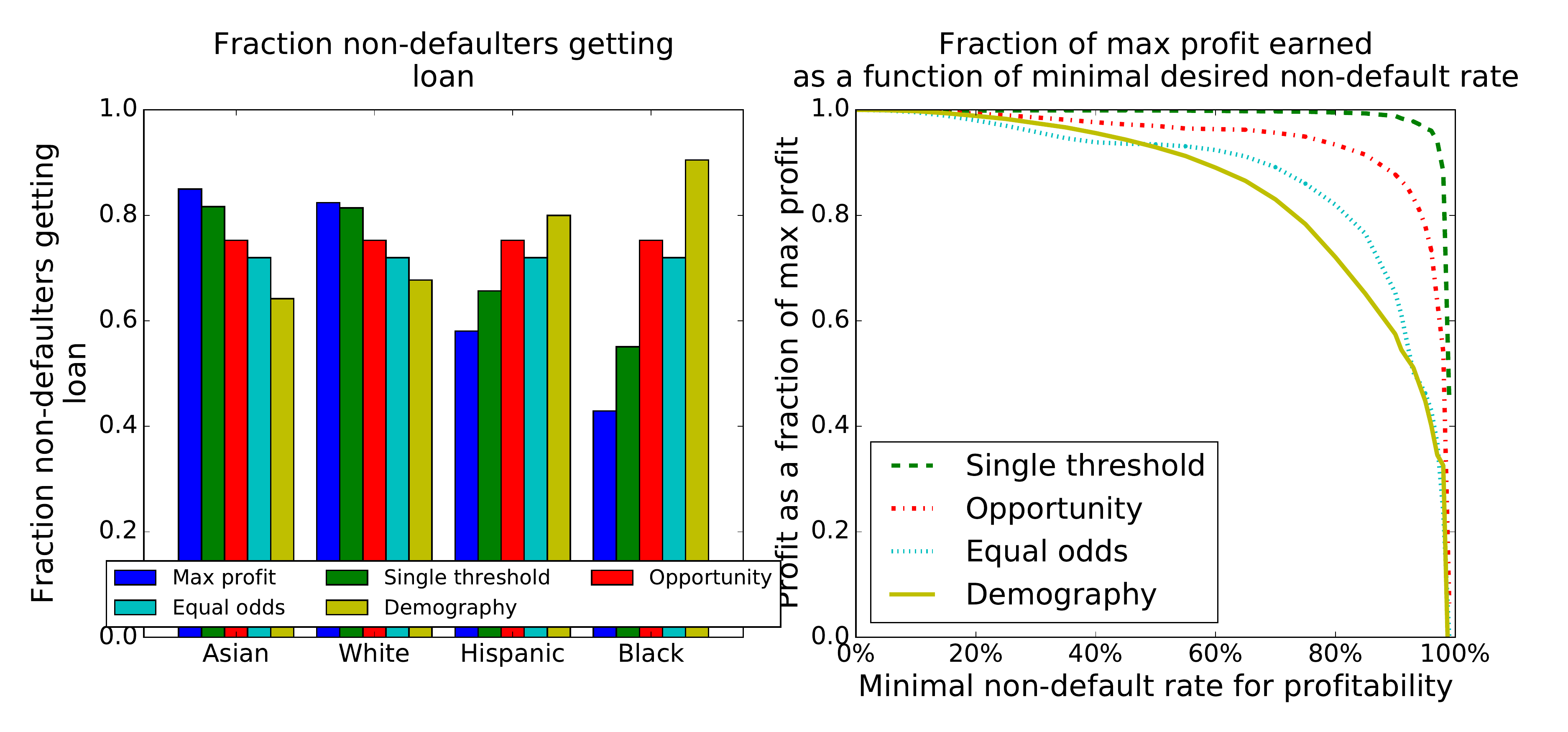}
  \caption{On the left, we see the fraction of non-defaulters that
    would get loans.  On the right, we see the profit achievable for
    each notion of fairness, as a function of the false
    positive/negative trade-off.}
  \label{fig:fico-profit}
\end{figure}

The right side of Figure~\ref{fig:fico-profit} gives the profit
achieved by each method, as a fraction of the max profit achievable.
We show this as a function of the non-default rate above which loans
are profitable (i.e. 82\% in the other figures).  At 82\%, we find
that a race blind threshold gets 99.3\% of the maximal profit, equal
opportunity gets 92.8\%, equalized odds gets 80.2\%, and demographic
parity gets 69.8\%.  So equal opportunity fairness costs less than a
quarter what demographic parity costs---and if the classifier
improves, this would reduce further.

The difference between equal odds and equal opportunity is that under
equal opportunity, the classifier can make use of its better accuracy
among whites.  Under equal odds this is viewed as unfair, since it
means that white people who wouldn't pay their loans have a harder
time getting them than minorities who wouldn't pay their loans.  An
equal odds classifier must classify everyone as poorly as the hardest
group, which is why it costs over twice as much in this case.  This
also leads to more conservative lending, so it is slightly harder for
non-defaulters of all groups to get loans.

The equal opportunity classifier does make it easier for defaulters to
get loans if they are minorities, but the incentives are aligned properly.
Under max profit, a small group may not be worth figuring out how to
classify and so be treated poorly, since the classifier can't identify
the qualified individuals.  Under equal opportunity, such
poorly-classified groups are instead treated better than
well-classified groups.  The cost is thus born by the company using
the classifier, which can decide to invest in better classification,
rather than the classified group, which cannot.  Equalized odds gives
a similar, but much stronger, incentive since the cost for a small
group is not proportional to its size.

While race blindness achieves high profit, the fairness guarantee is
quite weak.  As with max profit, small groups may be classified poorly
and so treated poorly, and the company has little incentive to improve
the accuracy.  Furthermore, when race is redundantly encoded, race
blindness degenerates into max profit.

\section{Conclusions}
We proposed a fairness measure that accomplishes two important desiderata.
First, it remedies the main conceptual shortcomings of demographic parity as a
fairness notion. Second, it is fully aligned with the central goal of supervised
machine learning, that is, to build higher accuracy classifiers.
In light of our results, we draw several conclusions aimed to help interpret and
apply our framework effectively.

\paragraph{Choose reliable target variables.}

Our notion requires access to observed outcomes such as default rates in the
loan setting. This is precisely the same requirement that supervised learning
generally has. The broad success of supervised learning demonstrates that this
requirement is met in many important applications. That said, having access
to reliable ``labeled data'' is not always possible. Moreover, the measurement
of the target variable might in itself be unreliable or biased. Domain-specific
scrutiny is required in defining and collecting a reliable target variable.

\paragraph{Measuring unfairness, rather than proving fairness.}

Due to the limitations we described, satisfying our notion (or any other
oblivious measure) should not be considered a conclusive \emph{proof of
fairness}. Similarly, violations of our condition are not meant to be a proof of
unfairness. Rather we envision our framework as providing a reasonable way of
discovering and measuring potential concerns that require further scrutiny. We
believe that resolving fairness concerns is ultimately impossible without
substantial domain-specific investigation. This realization echoes earlier
findings in ``Fairness through Awareness''~\cite{DworkHPRZ12} describing the
task-specific nature of fairness.

\paragraph{Incentives.}

Requiring equalized odds creates an incentive structure for the entity
building the predictor that aligns well with achieving fairness.
Achieving better prediction with equalized odds requires collecting
features that more directly capture the target $Y$, unrelated to its
correlation with the protected attribute. Deriving an equalized
odds predictor from a score involves considering the pointwise minimum
ROC curve among different protected groups, encouraging constructing
of predictors that are accurate in all groups, e.g.,~by collecting data
appropriately or basing prediction on features predictive in all
groups.

\paragraph{When to use our post-processing step.}
An important feature of our notion is that it can be achieved via a simple and
efficient post-processing step. In fact, this step requires only aggregate
information about the data and therefore could even be carried out in a
privacy-preserving manner (formally, via Differential Privacy). In contrast,
many other approaches require changing a usually complex machine learning
training pipeline, or require access to raw data. Despite its simplicity, our
post-processing step exhibits a strong optimality principle. If the underlying
score was close to optimal, then the derived predictor will be close to optimal
among all predictors satisfying our definition. However, this does not mean that
the predictor is necessarily good in an absolute sense. It also does not mean
that the loss compared to the original predictor is always small.  An
alternative to using our post-processing step is always to invest in better
features and more data. Only when this is no longer an option, should our
post-processing step be applied.

\paragraph{Predictive affirmative action.}
In some situations, including Scenario II in
Section~\ref{sec:oblivious}, the equalized odds predictor can be thought of as
introducing some sort of affirmative action: the optimally predictive
score~$R^*$ is shifted based on~$A$. This shift compensates for the fact that,
due to uncertainty, the score is in a sense more biased then the target label
(roughly, $R^*$ is more correlated with~$A$ then~$Y$ is correlated with~$A$).
Informally speaking, our approach transfers the \emph{burden of uncertainty}
from the protected class to the decision maker. We believe this is a reasonable
proposal, since it incentivizes the decision maker to invest additional
resources toward building a better model.

\bibliographystyle{alpha}
\bibliography{refs}

\end{document}